\documentclass[letterpaper, 10 pt, conference]{ieeeconf}
\IEEEoverridecommandlockouts
\usepackage{cite}
\usepackage{subcaption,titlesec}
\usepackage{amsmath,amssymb,amsfonts}

\newtheorem{model}{Model} 
\newtheorem{proposition}{Proposition} 
\newtheorem{lemma}{Lemma} 

\renewcommand{\QED}{\QEDopen}

\usepackage{graphicx}
\usepackage{standalone}
\graphicspath{{./sections/figs/}}

\newcommand{\sensor}{touch sensor} 
\newcommand{\sensors}{touch sensors}

\newcommand{\Tsr}{{T}_{sr}}
\newcommand{\Trw}{{T}_{rw}}
\newcommand{\robot}{\mathcal{R}}
\newcommand{\reals}{\mathbb{R}}
\newcommand{\delW}{\partial W}
\newcommand{\qRob}{q_r}
\newcommand{\histx}{\tilde{x}}
\newcommand{\histX}{\tilde{X}}

\newcommand{\CD}{cd}
\newcommand{\CDS}{cds}
\newcommand{\CS}{cs}
\newcommand{\DS}{ds}
\newcommand{\BC}{bc}
\newcommand{\projFn}{\texttt{proj}}

\usepackage{acronym}
\acrodef{cw}[cw]{clockwise}
\acrodef{ccw}[ccw]{counterclockwise}

\begin{document}
\bstctlcite{IEEEexample:BSTcontrol} 

\title{Visibility-Inspired Models of Touch Sensors for Navigation}

\author{Kshitij Tiwari${}^{1}$, Basak Sakcak${}^1$, Prasanna Routray${}^2$, Manivannan M.${}^2$, and Steven M. LaValle${}^1$
\thanks{${}^1$Center of Ubiquitous Computing, Faculty of Information Technology and Electrical Engineering, University of Oulu, Oulu, Finland,
        {\tt\small firstname.lastname@oulu.fi}}%
\thanks{${}^2$Touch Lab, Center for Virtual Reality and Haptics, Indian Institute of Technology Madras, India}
\thanks{This work was supported by a European Research Council Advanced Grant (ERC AdG, ILLUSIVE: Foundations of Perception Engineering, 101020977), Academy of Finland (projects PERCEPT 322637, CHiMP 342556), and Business Finland (project HUMOR 3656/31/2019).}}

\maketitle


\begin{abstract}
This paper introduces mathematical models of \sensors\ for mobile robots based on visibility.  Serving a purpose similar to the pinhole camera model for computer vision, the introduced models are expected to provide a useful, idealized characterization of task-relevant information that can be inferred from their outputs or observations. Possible tasks include navigation, localization and mapping when a mobile robot is deployed in an unknown environment. These models allow direct comparisons to be made between traditional depth sensors, highlighting cases in which touch sensing may be interchangeable with time of flight or vision sensors, and characterizing unique advantages provided by touch sensing. The models include contact detection, compression, load bearing, and deflection.  The results could serve as a basic building block for innovative touch sensor designs for mobile robot sensor fusion systems. 
\end{abstract}

\section{Introduction}
\label{sec:intro}

\begin{figure*}[!htbp]
\centering
\begin{subfigure}{.2\textwidth}
  \centering
  \includegraphics[scale=0.15, clip=true, trim ={3cm 1cm 0cm 0cm}]{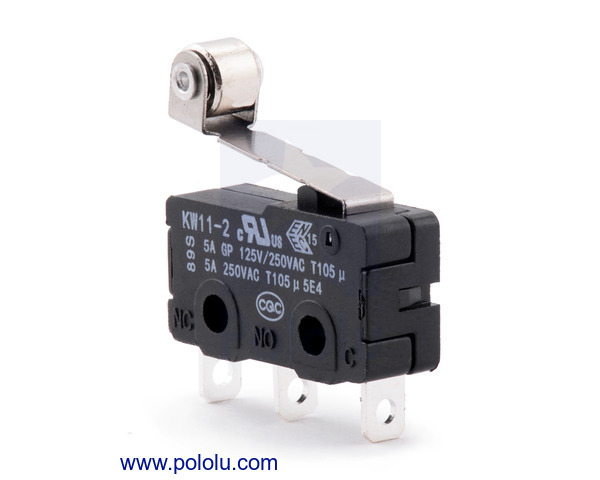}
  \caption{Roller lever switch}
  \label{fig:bumper-switch-1}
\end{subfigure}%
\begin{subfigure}{.2\textwidth}
  \centering
  \includegraphics[scale=0.11]{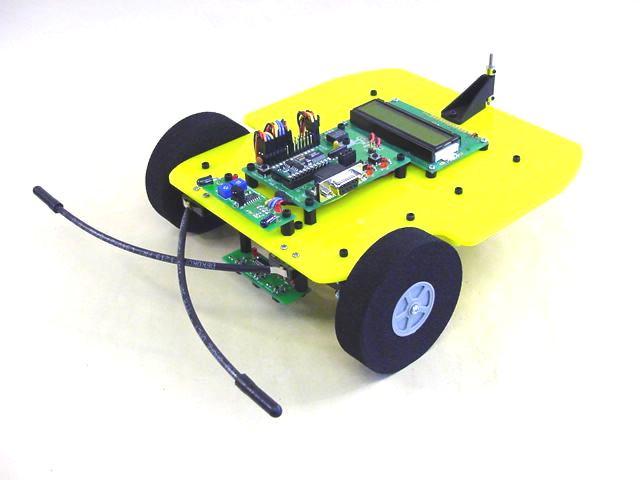}
  \caption{Bumper switch}
  \label{fig:bumper-switch-2}
\end{subfigure}%
\begin{subfigure}{.2\textwidth}
  \centering
    \includegraphics[scale=0.16]{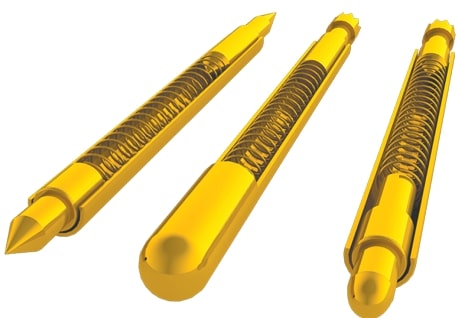}
    \caption{Pogo pin}
    \label{fig:pogopin}
\end{subfigure}%
\begin{subfigure}{.2\textwidth}
  \centering
  \includegraphics[scale=0.12]{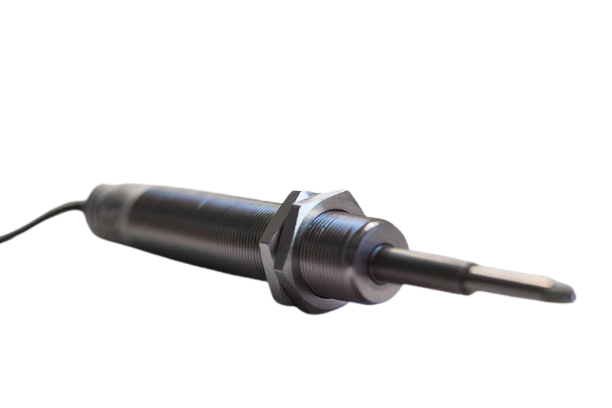}
    \caption{LVDT}
    \label{fig:lvdt}
\end{subfigure}%
\begin{subfigure}{.2\textwidth}
  \centering
    \includegraphics[scale=0.16]{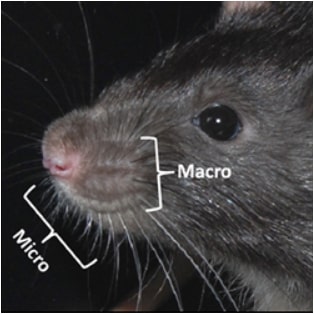}
    \caption{Rat vibrissae}
    \label{fig:macro-micro-vibrissae}
\end{subfigure}%
\caption{Various types of \sensors. \ref{fig:bumper-switch-1} and \ref{fig:bumper-switch-2} are common contact detectors; \ref{fig:pogopin} and \ref{fig:lvdt} are typical compressible \sensors\ (Images from e-ptnk.co.kr/eng/pogo-pin/ and hoffmann-krippner.com/types-of-lvdts/, respectively). LVDT stands for Linear Variable Differential Transducer; \ref{fig:macro-micro-vibrissae} shows macro- and microvibrissae in rats~\cite{grant2012role}.}
\end{figure*}

Touch sensor designs have evolved over the past four decades yet the touch modality is not as deeply explored as vision and auditory. The existing touch sensors are often inspired by touch receptors and are aimed at robotic and industrial applications which typically require human-like prehensile manipulation capabilities for the robots to work safely around humans.
Beyond manipulation, in nature, mammals (like rats, shrews, pinnipeds etc.) and fishes use the touch modality for navigation utilizing pre- and post-contact touch feedback. Thus, in this work, we look further to explore the possibility of using the touch modality for navigation with mobile robots. 

For navigation in unknown and unstructured environments researchers typically resort to distal sensors such as cameras and Lidars. The utility of such sensors may seem limiting when it comes to cases where a robot needs to navigate in dimly lit environments. Similarly, soft terrains present wheel slippage and sinkage challenges for ground robots and the space exploration community can benefit from such solutions~\cite{reina2006wheel}. Although vision-based solutions such as \cite{hegde2013computer} exist, these approaches are often prone to failure under sudden illumination changes and dust. Similarly, when inspecting underwater mines using an autonomous marine vehicle, the murky and rough waters may even harm the optical sensors~\cite{scholz2004profile}. Few researchers have started looking into potential use cases of vibrissae-enabled robots~\cite{an2021biomimetic}. Rodents use their vibrissae, a form of touch sensor, to detect both tactile and kinesthetic features \cite{mccloskey1988kinesthesia}. Although tactile features help infer the surface profile through textural perception~\cite{grant2012role}, kinesthetic features are perceived through proprioception~\cite{tuthill2018proprioception}, which can be thought of as self-localization. 

In this work, we investigate the utility of \sensors\, mounted on a mobile robot for proprioception. Some \sensors\, provide very limited information such as the bumper sensors (see Figs.~\ref{fig:bumper-switch-1}, \ref{fig:bumper-switch-2}) whereas others could capture a lot more information about the objects such as their surface profile. As surface profile provides a qualitative idea about the environment, profile sensing with \sensors\ can augment the path planning capability for mobile robots. We consider two types of \sensors- \textit{rigid} and \textit{compliant} where compliance is limited to their ability to compress (change its volume under applied transverse load, as in Figs.~\ref{fig:pogopin}, \ref{fig:lvdt}) or bend (similar to the vibrissae shown in Fig.~\ref{fig:macro-micro-vibrissae}). Compressible \sensors\ can be used for motion planning for collision resilient robots as was shown in~\cite{lu2021deformation} and hence, we consider compressibility as a desirable form of compliance for \sensors. Also, we are interested in contrasting the quality of information between actuated and unactuated \sensors. For instance, in \cite{giguere2011simple}, an actuated probe (rigid link) was used for terrain identification to guide the otherwise blind robot through an environment while avoiding dangerous patches. In \cite{lepora2018tacwhiskers}, the researchers showed how actuation directly impacts object localization. They considered two variants of a whisker-like touch sensor array- one actuated like the motile macrovibrissae and the other static like the immotile microvibrissae. They showed that the static \sensor\, had relatively poor object localization accuracy, and in this work we analyze such claims in terms of preimages.

To contrast various \sensors, we resort to \textit{virtual sensor models} from \cite{lavalle2012sensing}. The aim here is to compare the sensors in terms of their preimages, and develop mathematical models independent of their physical realization. 
This could elucidate various aspects of \sensors\ to increase the task success rate when deploying mobile robots in unstructured environments for tasks such as navigation, localization or mapping.

\section{Virtual \sensor~models}
\label{sec:virtual-sensors}

There is a wide spectrum of \sensors~available in the literature, some of which were discussed above. But the question remains, given any two \sensors, how does one know which is superior compared to the other? A coarse elimination is possible given some prior knowledge of the downstream task, but then given a family of sensors suited to the task, how does one narrow down to an optimal sensor with respect to a particular relevant criterion? Every transducer and design has strengths and weaknesses making it difficult to do a fair comparison. To facilitate this discussion, here we discuss the concept of \textit{virtual sensors} from \cite{lavalle2012sensing}. A \textit{virtual sensor} is a mathematical abstraction of the sensor, different from the models that explain its physics, and independent of its physical realization thereby allowing a fair comparison. In what follows, we will first describe the concept of sensor mapping followed by formally defining the state spaces. Then, we present various virtual \sensors\ followed by a discussion on compositions of such models and their comparison with conventional visibility-based sensors.\newline


\subsection{Sensor Mapping}\label{sec:sm}
Even though the mathematical model that explains the physics of a sensor varies with its design we can define a virtual sensor \cite{lavalle2012sensing} as a mapping, 

\begin{equation}
    h : X \rightarrow Y,
    \label{eq:sensor_mapping}
\end{equation}
from the robot's state space $X$ to the observation space $Y$. The state space $X$ refers to the set of all the states that the robot can be in and the observation space $Y$ is the set of all possible observations (measurements) that a sensor can make. In the following, we will refer to $h$ as the sensor mapping. In order to lighten the notation, in the following sections, when describing some sensors, we will use $h(x;m)$ to describe a sensor mapping for which $m$ are the parameters. 

Typically a sensor mapping is not bijective; whereas each $x \in X$ maps to a single $y \in Y$, multiple states can give out the same observation. 
Therefore, in the general case, the sensor mapping is not invertible and the preimage of $y$ under $h$, written as $h^{-1}(y)$, is defined as the set,
\begin{equation} 
    h^{-1}(y) = \{x \in X \ | \ y = h(x)\}.
    \label{eq:preimage-defn}
\end{equation}
Note that if the sensor mapping is bijective, then Eq.~\eqref{eq:preimage-defn} corresponds to the inverse of the function $h$. As $h$ is defined over the state space $X$, the subsets of $X$ corresponding to the preimages of $h$ form a partition of $X$ denoted by $\Pi(h)$.

Consider two sensors, $h_1$ and $h_2$, defined over the same fixed state space $X$ that has different observation spaces $Y_1$ and $Y_2$. 
If the partition $\Pi(h_1)$ is a refinement of $\Pi(h_2)$ then $h_1$ can simulate $h_2$ which means that there exists a function $g: Y_1 \rightarrow Y_2$ such that $h_2(x)=g(h_1(x))$, written as $h_2=g\circ h_1$ \cite{lavalle2012sensing}.
Two sensors are equivalent, denoted by $h_1 \cong h_2$, if the partition induced by the sensors, $\Pi(h_1)$ and $\Pi(h_2)$, respectively, satisfies $\Pi(h_1)=\Pi(h_2)$. 
If two sensors are equivalent they can simulate each other.

\subsection{State Space}\label{sec:state_space}

We consider a mobile robot moving in a 2D planar environment $W$  which is the closure of the open set that is piecewise diffeomorphic to a finite disjoint union of circles. The boundary of $W$ is denoted as $\delW$. The robot is equipped with a \sensor\, which can be seen as a link attached to the robot base whose characteristics vary across different models. However, the discussion in this work is limited to \sensors\, that are either rigid or compliant that too are either compressible or bendable. 

We will refer to three coordinate frames and the corresponding homogeneous transformations between them as shown in Fig~\ref{fig:coordinate-transforms}a. We express the configurations of the \sensor\ and the robot as $q_s$ and $q_r$, respectively. The robot configuration $\qRob$ is defined by $q_r=(q_x,q_y,q_\theta)$, in which $(q_x,q_y) \in W$ and $q_\theta \in S^1$ are the position and the orientation of the robot with respect to the world frame, respectively.  The world frame corresponds to an absolute reference frame with respect to which the robot configuration is defined. The robot frame and the sensor frame are the reference frames attached to the robot and the \sensor, respectively. The homogeneous transform between the robot frame and the world frame is given by $\Trw(\qRob)$ and varies with the configuration of the robot $\qRob$. As for the \sensor\, in its nominal form, it is aligned with the x-axis of the sensor frame (see Figs.~\ref{fig:coordinate-transforms}b and c) which will be referred as the sensor axis from hereon. 
For unactuated \sensors, the sensor frame is fixed with respect to the robot frame and $\Tsr$ is the homogeneous transform between two coordinate frames. To simplify the notation, we will assume that the robot frame and the sensor frame coincide for unactuated \sensors, such that $\Tsr$ is an identity matrix. In cases that the \sensor\ is actuated, the homogeneous transform between the sensor and robot frame is given by $\Tsr(\psi)$ in which $\psi \in S^1$ is the amount of rotation of the sensor frame with respect to the robot frame.

We now define the configurations of compliant \sensors. Compressible \sensors\, can compress along the sensor axis, similar to a linear spring or a link attached to a prismatic joint (like Figs.~\ref{fig:pogopin} and \ref{fig:lvdt}). The sensor configuration is expressed by the variable $q_s=\delta$ and the robot state is the tuple $x=(\qRob,\delta)$ in which $\delta \in [0, \delta_{max}]$ is the amount of compression along the sensor axis (as shown in Fig.~\ref{fig:coordinate-transforms}b).
For the bendable \sensors\, (like Figs.~\ref{fig:bumper-switch-2} and \ref{fig:macro-micro-vibrissae}) we assume that the deflection from the nominal form is parameterized by $q_s=\alpha \in [\alpha_{min},\alpha_{max}]$ in which $\alpha_{min}<0<\alpha_{max}$. 
If $\alpha=0$ then the \sensor\, is in its nominal form (see Fig.~\ref{fig:coordinate-transforms}c). Thus, each $\alpha$ corresponds to a diffeomorphism that maps the points along the \sensor\, at its nominal shape to the ones corresponding to $\alpha\not=0$. The robot state $x$ is the robot configuration and the sensor configuration, i.e., $x=(\qRob,\alpha)$. 

If the sensor can not change its configuration by itself but conforms to the environment, $\delta$ and $\alpha$ are not free variables. Consequently, their value can be derived from the robot configuration $\qRob$ and the environment description $W$, similar to closed kinematic chains. To this end, for each type of sensor, we can define a contact-constraint equation. For a compressible \sensor, 
\begin{equation}\label{eq:constraint_comp}
    \delta=f_{comp}(\qRob,\delW)
\end{equation}
is the amount of compression for a given environment and robot configuration. Similarly,
\begin{equation}\label{eq:constraint_bend}
    \alpha=f_{bend}(\qRob,\delW)
\end{equation}
describes the amount of deflection. 
The solution to Eq.~\eqref{eq:constraint_comp} is unique, meaning that for each $\qRob$ we can uniquely determine the amount of compression. However, Eq.~\eqref{eq:constraint_bend} can admit multiple solutions.

Finally, let $\robot(x) \subset \reals^2$ denote the points occupied by the robot in the world frame, equipped with the \sensor\, at state $x$.
The robot state space $X$ is the set of all the states that the robot can be in without collisions. The robot is in collision if $\robot(x) \setminus W \not=\emptyset$.

\begin{figure}[!t]
    \centering
    \includegraphics[scale=0.3]{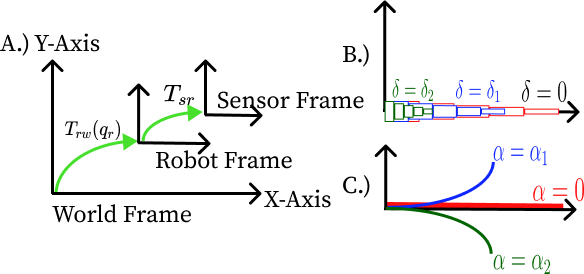}
    \caption{(a): Transformation between the world (fixed), robot and sensor frames. The function $T$ represents a homogeneous transformation. The transformation from robot frame to world frame is given by $\Trw(\qRob)$, and the static transform from the sensor frame to robot frame is given by $\Tsr$. For actuated \sensors\, the sensor to robot frame transform is given by $\Tsr(\psi)$ in which $\psi$ represents the rotation of the sensor frame with respect to the robot frame. (b) Various configurations of compressible and (c) bendable  \sensor.}
    \label{fig:coordinate-transforms}
\end{figure}

\subsection{Sensor models}
\label{subsec:sensor-models}

We now introduce a family of \sensors\, spanning rigid and compliant \sensors. Although the discussion focuses on a single unactuated
\sensor, the models presented herein can easily be extended to actuated \sensors\, 
as will be discussed in Sec.~\ref{sec:analysis}. 

We first consider contact detectors such that the sensor is not compliant with the environment and it is fixed with respect to the robot (see, for example, Fig.~\ref{fig:bumper-switch-1}). The state in this case can be expressed solely by the robot configuration, i.e., $x=\qRob$. 

\begin{model}[Contact detector]
Let $p\in \reals^2$ be the position of the contact detector whose coordinates are expressed in the sensor
frame. The sensor mapping for the contact detector is 
\label{model:contact-detector}
\begin{align}\label{eq:contact-detector} 
h_{\CD}(x;p) &= 
    \begin{cases}
      1& \text{if } \Trw(\qRob)\Tsr\, \begin{pmatrix} p & 1 \end{pmatrix}^T\in \delW\\ 
      0& \text{otherwise}.  
    \end{cases}
\end{align}
\end{model}
The preimages corresponding to $h_{\CD}$ result in a partition $\Pi(h_{\CD})$ of two classes.

The next model describes a \sensor\, when the contact can be detected along the \sensor\, whose geometry is defined as the curve $\sigma : [0,1] \rightarrow \reals^2$ which maps to coordinates in the sensor
frame. 

\begin{model}[Contact detector strip]
\label{model:contact-detector-strip}
Suppose there are $N$ interior disjoint intervals, $I_i$, $i=1,\dots,N$ such that $\bigcup_{i=1}^N I_i = [0,1]$. Let the end points of $N$ intervals correspond to a uniform discretization of $[0,1]$ imposing a particular resolution. Let $\mathcal{I}$ denote the set of $N$ intervals and let $\mathcal{I}_C$ be the set of intervals that are in contact with the boundary such that $\mathcal{I}_{C}=\{I \in \mathcal{I} \mid \exists p\in \sigma(I) \text{ s.t. }\, h_{\CD}(x,p)=1 \} \subseteq \mathcal{I}$. Then, the contact detector strip is described by 
\begin{equation}
    h_{\CDS}(x,\sigma)=\mathcal{I}_C.
\end{equation}

For the limiting case, as $N\rightarrow \infty$, the observation corresponds to the exact location (or locations) of contact along the strip. 
At the other extreme, when $N=1$, the whole strip reduces to just one interval reporting whether a contact is made anywhere along the strip.
\end{model}

Next, we discuss a \sensor~that complies in response to touch by compressing along the sensor axis. The robot state is given by $x=(\qRob,\delta)$ in which $\delta \in [0, \delta_{max}]$ is the amount of compression. 

\begin{model}[Compression sensor]\label{model:compression-sensor}
The sensor which can detect the amount of compression in the \sensor\, is described by 
\begin{equation}
    h_{\CS}(x)=\delta.
\end{equation}
The preimage of an observation $y=h_{\CS}(x) > 0$ is typically a union of disjoint two-dimensional subsets of the state space. 
Suppose the length of the uncompressed \sensor\, is $\ell$. 
Let $P_{\theta} \subset W$ be the set of all points $u \in W$ for which there exists a $v\in\delW$ such that a line segment $\overline{uv}$ of length $l-\delta$ and slope $\theta$ satisfies $\overline{uv}\setminus W = \emptyset$.
Then $h_{\CS}^{-1}(y)=\bigcup_{\theta \in S^1}P_{\theta}\times \{\theta\} \times \{y\}$. The preimage of $y=0$ is all the states such that the \sensor\, of length $\ell$ is not in collision. Thus, the states for which the \sensor\, is in contact with the boundary and the ones such that the \sensor\, lies completely in the interior of $W$ belong to the same equivalence class and are indistinguishable.

\end{model}

Yet another form of compliance is the ability to bend under load. For this, consider the bendable \sensors\, for which the state is $x=(\qRob,\alpha)$.

\begin{model}[Load sensor]
\label{model:load-sensor}
This sensor can detect the amount of shear force applied to the compliant bendable \sensor. This model addresses the transverse point load, a force applied at a single point along the \sensor\,. The sensor mapping is described by
\begin{equation}
    h_{ls}(x)= F,
\end{equation}
and $F \in [-F_{max}, F_{max}]$, in which $F_{max}$ is the maximum bearable load.
\end{model}

The next model describes the state of the robot in terms of $x=(\qRob,\alpha)$, in which $\alpha$ is the deflection of the \sensor\ upon contact. The load applied, physical dimension, and material property together determine the deflection.

\begin{model}[Deflection sensor]
\label{model:deflection-sensor}
The sensor which can detect the amount of deflection in the \sensor\, is described by
\begin{equation}
    h_{\DS}(x)=\alpha.
\end{equation}

For each $\alpha \in [\alpha_{min},\alpha_{max}]$, the \sensor\, has the shape described by the curve $\sigma_\alpha:[0,1]\rightarrow\reals^2$ which maps to the coordinates expressed in the sensor frame. The preimages for $y=h_{\DS}(x) \neq 0$ is all states such that intersection of $\delW$ and the image of $\sigma_\alpha$ (coordinates mapped to the world frame) is not an empty set and that $\robot(x) \setminus W=\emptyset$. This corresponds to a union of disjoint 2D or 3D {(in case multiple orientations are possible to touch the boundary at some point along the sensor)} subsets of the state space. Also for Model~\ref{model:deflection-sensor}, for $y=0$, the states for which the \sensor\, is in contact with the boundary and the ones such that the \sensor\, lies completely in the interior of $W$ belong to the same equivalence class and are indistinguishable.

\end{model}

\subsection{Composition of virtual \sensors}

Recall that the preimage corresponding to the observation $y=0$ for Models~\ref{model:compression-sensor} and \ref{model:deflection-sensor} include also the states at which the sensor is not in contact with the environment which creates ambiguity. This can simply be alleviated by adding a contact detector at the \sensor\, tip. To this end, Model~\ref{model:compression-sensor} can be extended by combining it with Model~\ref{model:contact-detector} such that
\begin{equation}\label{eq:comp_contact}
    h'_{\CS}(\qRob,\delta)=
    \begin{cases}
        \delta & \text{if }\, h_{\CD}(\qRob;p_t)=1 \\ 
        \# & \text{otherwise},
    \end{cases}
\end{equation}
in which $\#$ refers to no value and $p_t$ is the point corresponding to the tip of the \sensor\, whose exact location in sensor frame is $p_t=(\ell-\delta,0)$. The same strategy can be followed also for bendable \sensor\, corresponding to a deflection sensor (Model~\ref{model:deflection-sensor}) in which case $p_t=\sigma_\alpha(1)$ and the respective sensor mapping is 
\begin{equation}\label{eq:comp_deflect}
    h'_{\DS}(\qRob,\alpha)=
    \begin{cases}
        (\alpha, 1) & \text{if }\, h_{\DS}(\qRob;p_t)=1 \\ 
        (\alpha, 0) & \text{otherwise}.
    \end{cases}
\end{equation}

On the other hand, the main issue with bendable \sensors\, is the ambiguity regarding the point of contact. Unlike compressible \sensors\, for which the contact needs to be at the tip to get a non-zero observation, bendable \sensors\, can be in contact with the environment at any point along the sensor and still achieve that. To refine the preimages, we can enhance $h_{\DS}$ (Model~\ref{model:deflection-sensor}) with a contact detector strip (Model~\ref{model:contact-detector-strip}). Suppose that the contact detector strip can conform to the shape assumed by the bendable \sensor\, (like an elastic strip). The new sensor model is described as
\begin{equation}\label{eq:model_contact_strip_bendable}
     h_{\BC}(\qRob,\alpha)= h_{\CDS}(\qRob;\sigma_\alpha),
\end{equation}
in which $\sigma_\alpha$ is the curve corresponding to the deflection determined by $\alpha = h_{\DS}(x)$. The observation vector $y=h(x)$ would then correspond to the intervals (or points if $K=\infty$) such that the bendable \sensor\, is in contact with $\delW$. The partition of $X$ resulting from the preimages of $h$, i.e., $\Pi(h_{\BC})$, will be a refinement of $\Pi(h_{\DS})$ such that each element of $\Pi(h_{\BC})$ is a union of disjoint two-dimensional subsets of $X$.

\begin{figure}
    \centering
    \includegraphics[scale=0.3]{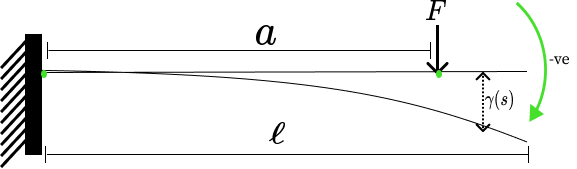}
    \caption{A bendable \sensor\ can be modeled as a simply supported cantilever beam of length $\ell$. An intermediate load $F$ is applied at a distance $a\leq\ell$ from the base in the positive direction of sensor axis. Then, $\gamma(s)$ is the deflection at any point $s$ along the sensor such that $\gamma(0)=0$ and maximum deflection is attained at $\ell$.}
    \label{fig:load-sensor}
\end{figure}

Suppose the bendable sensor is in contact with the environment at a single point and let $h_{\BC}$ be the contact detector strip defined in Eq.~\eqref{eq:model_contact_strip_bendable} that reports the exact point of contact ($N=\infty$). Let $a=h_{\BC}(x)$ be the point that is in contact with the boundary (coordinates expressed in the sensor frame), and let $F=h_{ls}(x)$ be the load acting on the sensor. If we consider only small deflections that satisfy the assumptions of classical beam theory, we can model the deflection from the nominal sensor shape as a function of the distance along the sensor length. Then, deflection $\gamma(s)$ (see Fig~\ref{fig:load-sensor}) at any $s \in [0,1]$ under load $F=h_{ls}(x)$ can be calculated as~\cite{shigley2011shigley}
\begin{align}
    \gamma(s)=
    \begin{cases}
        \frac{- F (s\ell)^2 (3 a - s\ell)}{6 E I} & \text{if }\, 0 \leq s < \frac{a}{\ell} \\
        \frac{- F a^2 (3 s\ell - a)}{6 E I} & \text{if }\, \frac{a}{\ell} \leq s < 1  \\
        \frac{F a^2 (3 \ell - a)}{6 E I} & \text{if }\, s = 1,  \\
    \end{cases}
\end{align}
in which $E\,,I$ stand for the Young's modulus and moment of inertia, respectively, and $\ell$ is the sensor length.
Note that the maximum deflection happens at the tip of the \sensor. 
By considering small deflections, we also assume that the projection of the sensor tip on the horizontal axis of the sensor frame corresponds to $\ell$ and that the length of the sensor under load, that is, $\ell'$ satisfies $\ell' \approx \ell$. Consequently, we can define the curve corresponding to the bendable \sensor\, as $\sigma_\alpha(s)=(s\ell, \gamma(s))$ in which $\alpha = h_{\DS}(x)$. This provides an example when the composition of a load sensor and a contact detector sensor in Eq.~\eqref{eq:model_contact_strip_bendable} is equivalent to a deflection sensor (Model~\ref{model:deflection-sensor}).

\section{Comparison to visibility-based models}
\label{sec:analysis}

In this section, we compare \sensors\ from Sec.~\ref{subsec:sensor-models} with traditional depth sensors from \cite{lavalle2012sensing} and evaluate conditions under which these are equivalent with respect to the definitions introduced in Sec.~\ref{sec:state_space}.

\subsection{Virtual depth sensors}\label{sec:depth_sensors}
Relevant depth sensor models are presented here for completeness. 
Let $Q=W \times S^1$ be the robot configuration space, that is, the set of all possible configurations of a point robot\footnote{Note that we do not need to assume a point robot, and $Q$ can as well be the set of all robot configurations such that the robot footprint is contained in $W$.}, over which a depth sensor is defined.

\begin{model}[Depth-limited directional depth sensor]
\label{model:lim-directional-depth}
Let $b(\qRob) \in \delW$ denote the point  
struck by the ray emanating from $(q_x,q_y)$ in the direction of $q_\theta$. The sensor model is given by
\begin{equation}
    h_{dds}(\qRob) = 
    \begin{cases}
        d(\qRob)& \text{if } d_{min} \leq d(\qRob) \leq d_{max}\\
        \# &\text{otherwise,}
    \end{cases}
\end{equation}
in which $d(\qRob)=||(q_x,q_y)-b(\qRob)||$
and $[d_{min},d_{max}]$ with $d_{min}<d_{max}$ is the sensor detection range.
\end{model}

Similar to Model~\ref{model:lim-directional-depth}, we can define a K-directional depth sensor.
Suppose there is a set of offset angles $\phi_1,\ldots,\phi_K$, which are oftentimes evenly spaced.
The sensor mapping for such a sensor is then given by $h_{kdd}(\qRob)=(y_1,\dots,y_K)$ in which $y_k=||(q_x,q_y)-b(q_x,q_y,q_\theta+\phi_k)||$ for every offset angle $\phi_{k \in \{1,\ldots,K\}}$.

\begin{model}[Depth-limited omnidirectional depth sensor]\label{model:lim-omnidir-depth}
In the limit case, as $K \rightarrow \infty$, we obtain an omnidirectional depth sensor. In this case, the observation $h_{od}(q_r)=y$ is an entire function $y: [0,2\pi) \rightarrow [d_{min},d_{max}]\cup \{\#\}$ and 
\begin{equation}
   y(\phi)= h_{dds}(q_x,q_y,q_\theta+\phi).
\end{equation}
\end{model}

\subsection{Unactuated \sensors}

For compressible and bendable \sensors, the state space $X$ is a subset of the Cartesian product of the robot configuration space and possible \sensor\, configurations. As the two families of sensors (depth sensors and \sensors) do not share the same state space, they cannot be compared directly based on the previously introduced definitions of equivalence and sensor simulation. 
However, in the following, we will use a map from the state space to the robot configuration space to argue about these relations. 
Let $\projFn:X \rightarrow Q$ be a function that maps the state to its elements corresponding to the robot configuration. 
Consider a \sensor\, $h_1$ and denote the partition induced by its preimages as $\Pi(h_1)$.
Suppose $\projFn$ is bijective. Then, the projection of the elements of each set in $\Pi(h_1)$ onto $Q$ forms a partition of $Q$, denoted as $\Pi_Q(h_1)$. 
A \sensor, $h_1$, defined over $X$ and a depth sensor, $h_2$, defined over $Q$ are equivalent if the partitioning of the robot configuration space induced by the two sensors, $\Pi_Q(h_1)$ and $\Pi(h_2)$, respectively, satisfy $\Pi_Q(h_1)=\Pi(h_2)$. Similarly, $h_1$ can {\em simulate} $h_2$ if they are equivalent or if $\Pi_Q(h_1)$ is a refinement of $\Pi(h_2)$. 


Since the contact constraint in Eq.~\eqref{eq:constraint_comp} admits a unique solution, the mapping $\projFn$ from the state space to robot configuration space is bijective for a compressible \sensor, that is, there is a single state that the robot can be in for each robot configuration. Hence, if the compressible \sensor\, is fully compressible, i.e., $\delta_{max}=\ell$, in which $\delta_{max}$ is the maximum attainable compression and $\ell$ is the sensor length, the projection of the elements in $\Pi(h_{\CS})$ to $Q$ forms a partition $\Pi_Q(h_{\CS})$ of $Q$. 
Note that, if the \sensor\, is not fully compressible, that is, $\delta_{max}<\ell$, then $\projFn(X) \subset Q$, meaning that there are some configurations that are not collision-free for a robot carrying a compressible \sensor.

We now consider a fully compressible sensor of length $\ell$ together with a contact detector at the sensor tip, that is, $h'_{\CS}$ as given in Eq.~\eqref{eq:comp_contact},
and a depth-limited directional depth sensor $h_{dds}$ (Model~\ref{model:lim-directional-depth}) with $d_{min}=0$ and $d_{max}=\ell$. The observation spaces $Y_{\CS}$ and $Y_{dds}$ corresponding to $h'_{\CS}$ and $h_{dds}$ are  $Y_{\CS}=Y_{dds}=[0,\ell]$. Suppose the \sensor\, frame is aligned with the robot frame such that $\psi=0$ and $\Tsr(\psi)$ is the identity matrix (see Fig.~\ref{fig:coordinate-transforms}). 
Then, we can define a function $g: Y_{\CS} \rightarrow Y_{dds}$ such that $h_{dds}(\qRob)=(g \circ h_{\CS})(\qRob,\delta)=\ell-h_{\CS}(\qRob,\delta)$ which implies that $\Pi_Q(h'_{\CS})=\Pi(h_{dds})$ and proves that the following proposition holds.

\begin{proposition} A fully compressible \sensor\, $h'_{\CS}$ of length $\ell$ as described in Eq.~\eqref{eq:comp_contact} 
can simulate (in the sense from Section \ref{sec:sm}) a depth-limited directional depth sensor with $d_{min}=0$ and $d_{max}=\ell$.
\label{prop:comp_dir_depth}
\end{proposition}

For the limiting case, as $\ell\rightarrow \infty$, a fully retractable compression sensor of length $\ell$ can simulate a directional depth sensor ($d_{max}=\infty)$, which returns the distance to the closest point on the environment boundary along the direction of $q_\theta$. At the other extreme, for $\ell=0$, $h_{\CS}$ is merely a contact detector $h_{\CD}$ (Model~\ref{model:contact-detector}) with $p=(0,0)$. 
Furthermore, it follows from Proposition~\ref{prop:comp_dir_depth} that an array of of $K$ fully compressible \sensors\, with length $\ell$, arranged at angles $\psi_1, \dots, \psi_K$ with respect to the robot reference frame such that simultaneous observations yield the vector $(y_1, \dots, y_K)$ 
can simulate a depth-limited K-directional depth sensor $h_{kdd}$ described in Sec.~\ref{sec:depth_sensors}.


For a bendable \sensor, the solution to the contact constraint in Eq.~\eqref{eq:constraint_bend} may not be unique. For instance, consider the scenario as shown in Fig.~\ref{fig:multi-obstacle-modes} in which \sensor\, can be in multiple deflection modes (corresponding to different $\alpha$) at the same configuration $q_r=(q_x,q_y,q_\theta)$. 
Thus, $\projFn$ is a many to one mapping and the projection of the elements of $\Pi(h_{\DS})$ to $Q$ corresponds to a cover of $Q$. This implies that we can not compare $h_{\DS}$ with another sensor defined over $Q$ in terms of previously introduced definitions of equivalence and sensor simulation.

\begin{figure}
    \centering
    \includegraphics[width=0.5\linewidth]{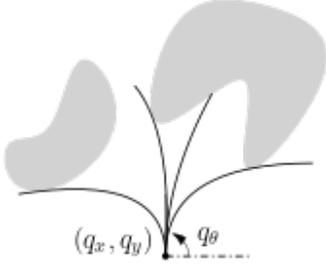}
    \caption{A bendable \sensor\, can be in multiple different modes at the same robot configuration. Therefore, the solution to the contact constraint is not unique.}
    \label{fig:multi-obstacle-modes}
\end{figure}

\subsection{Actuated \sensors}

Suppose that the \sensor\, is attached to a revolute joint that can be controlled to change its configuration. 
The state of the system $x$ now includes the robot configuration $\qRob$ and the joint configuration $\psi \in S^1$ such that $x=(\qRob,\psi,\dots)$.
A measurement is obtained by changing the configuration of the revolute joint while the robot configuration $\qRob$ is kept constant. Let a state trajectory up to time $t$ be denoted as $\histx: [0,t] \rightarrow X$. The set of all state-trajectories such that $\projFn(x)$ is constant for all $x \in \histx$ for any possible $t>0$ is denoted by $\histX$. The sensor mapping for an actuated sensor is $h: \tilde{X} \rightarrow Y$.
Since the state is now augmented with the joint configuration, in the following, with a slight abuse of previously introduced definitions, we consider the sensor models in Section~\ref{subsec:sensor-models} to be defined over this extended state space. 

We begin by considering a compressible \sensor\, described as $h'_{\CS}$ in Eq.~\ref{eq:comp_contact} attached to a revolute joint. The state of the system is expressed by $x=(\qRob,\psi,\delta)$.
Suppose that for each measurement, the revolute joint trajectory corresponds to a full rotation of the sensor that is defined as $\tilde{\psi}:[0,t] \rightarrow S^1$ such that $\tilde{\psi}$ is a  monotonically increasing function with $\tilde{\psi}(0)=0$ and $\tilde{\psi}(t)=2\pi$ while $\projFn(\tilde{x}(\tau))=\qRob$, $\forall \tau \in [0,t]$.

Let $\histX_{fr}\subseteq \histX$ be the set of all trajectories that satisfy this motion.
It follows from Proposition~\ref{prop:comp_dir_depth} that $h'_{\CS}$ can be used to obtain the distance to the boundary (up to $\ell$) along the sensor axis, which corresponds to the direction of $q_\theta+\psi$ for an actuated \sensor. 
Then, $\forall \psi \in [0,2\pi)$ the shifted distance $d(q_x,q_y,q_\theta+\psi)=\ell-h'_{\CS}(x)$. Considering the full motion that spans all $\psi \in [0,2\pi)$, $h_{\CS}$ combined with motion results in the observation $y:S^1 \rightarrow [0,\ell]$ such that $y(\psi)=d(q_x,q_y,q_\theta+\psi)$ if $d(q_x,q_y,q_\theta+\psi)\leq \ell$ and $y(\psi)=\#$, otherwise. Then, the following proposition holds. 

\begin{proposition}\label{prop:comp_omni_dir}
Suppose $\histX_{fr}$ corresponding to the full rotation of the revolute joint is not empty. Let $Q'\subseteq Q$ be defined as $Q' = \bigcup_{\histx \in \histX_{fr}}\projFn(\histx(0)).$ Then, an actuated compressible sensor as defined above can simulate (in the sense from Section \ref{sec:sm}) a depth-limited omnidirectional depth sensor (Model~\ref{model:lim-omnidir-depth}) defined over $Q'$. 
\end{proposition}
The main reason why an actuated compressible \sensor\,(even if it is fully compressible) may not simulate a depth-limited omnidirectional depth sensor defined over the entire $Q$ is that for certain $\qRob$ a full rotation of the revolute joint may not be possible due to the sensor characteristics and the environment (see for example, Fig.~\ref{fig:compression-issue}). This issue could be alleviated to some extent (although not fully) by defining a slightly more complicated motion (for example, a full (or partial) \ac{ccw} followed by a rotation in the \ac{cw} direction) or with the sensor design. To this end, configuration space decomposition of a robot carrying a compressible sensor and determining its connectivity remains as an interesting open problem. 

We now consider a mobile robot equipped with a bendable \sensor\, composed of a bendable link together with a deflection sensor (Model~\ref{model:deflection-sensor}) attached to a revolute joint. The state of the system is expressed as $x=(\qRob,\psi,\alpha)$. 
Considering the same motion corresponding to a full rotation of the revolute joint, let $P \subset W$ be part of the environment swept by the \sensor\, such that 
\begin{equation}\label{eq:swept_area_def_sensor}
P=\bigcup_{x\in \tilde{x}}\bigcup_{s\in[0,1]}\Trw(\qRob)\Tsr(\psi)\sigma_\alpha(s).    
\end{equation}
The observation is the function $y(\psi)=h_{\DS}(x)$. The preimage of $y$ contains all possible robot configurations $\qRob$ such that same profile information is obtained, i.e., all $\qRob$ such that $P \in W$ and $\partial P \cap \delW \not=\emptyset$, in which $\partial P$ is the boundary of the area swept by the \sensor. Note that the preimages corresponding to the combination of rotary motion and deflection sensor still result in a cover of the robot configuration space since different $P$ can be obtained for the same $\qRob$ based on the configuration of the bendable \sensor\, at the beginning of the motion.

We now consider a motion similar to actuated whisking (see for example \cite{prescott2009whisking, lepora2018tacwhiskers}) which corresponds to a full \ac{ccw} rotation followed by a full \ac{cw} rotation. In the following, we will show that with this type of motion, under certain conditions, a bendable \sensor\, with a deflection sensor can simulate a depth sensor.
Let $\mathcal{D}_\ell(q_x,q_y)$ denote the disk of radius $\ell$ centered at $(q_x,q_y)$ as shown in Fig.~\ref{fig:actuated-whisking}. Suppose that $\mathcal{D}_\ell(q_x,q_y) \setminus W$ (light gray area in Fig.~\ref{fig:actuated-whisking}) is convex and that at the beginning of the motion, the \sensor\, is in its nominal form, that is, $\alpha=0$. Denote the intersection point of the \sensor\, at the beginning of the motion and the boundary of $\mathcal{D}_\ell(q_x,q_y)$ as $p_0$. The instances at which $\alpha$ changes from zero to non-zero or non-zero to zero are called critical instances and they happen when the \sensor\, starts bending or regains its nominal form. 
Since $\mathcal{D}_\ell(q_x,q_y) \setminus W$ is convex, a rotation in single direction results in at most two critical instances: when the sensor gets in contact with the boundary of $W$ so it starts bending and if possible, when it regains its nominal shape. Suppose, during a full rotation in the \ac{ccw} direction, the nominal shape is achieved after bending and the swept area is denoted by $P_{ccw}$.
Let $p_1, p_2 \in \partial\mathcal{D}_\ell(q_x,q_y)$ be the intersection points of the \sensor\, and $\partial\mathcal{D}_\ell(q_x,q_y)$ at the beginning and at the end of the bending, respectively (see Fig.~\ref{fig:actuated-whisking}). We assume that $P_{ccw} \cap \partial W$ is a connected curve. Let $u_1$ be an endpoint of this curve which corresponds to the last point along the $\partial W$ that the \sensor\, is in contact with during a \ac{ccw} motion. For \ac{cw} rotation, denote the swept area by $P_{cw}$ and define $p_3, p_4$, similar to $p_1,p_2$, and define $u_2$ similar to $u_1$. We again assume that $P_{cw} \cap \partial W$ is connected. We say that $p'$ lies to the left of $p$ if $p'$ lies on the left hand side of the ray starting at $(q_x,q_y)$ and passing through $p$, the relation is denoted by $p' >_\ell p$.

Let $V(q_x,q_y)\subseteq W$ be the set of points visible from $(q_x,q_y)$. {A point $p'\in W$ is visible from $p \in W$ if the line segment $\overline{uv}$ is contained in $W$.} The following lemma establishes a relation between the visibility of $(q_x,q_y)$ and a bendable sensor at $(q_x,q_y)$ executing a whisking motion as described above which satisfies the assumptions.

\begin{lemma}\label{lemma:bendable_visibility}
Suppose $\mathcal{D}_\ell(q_x,q_y) \setminus W$ is convex,
and suppose that $u_1 >_\ell u_2$ and $p_0>_\ell p_2$. Then, $P = \mathcal{D}_\ell(q_x,q_y) \cap V(q_x,q_y)$, in which $P=P_{ccw} \cup P_{cw}$.
\end{lemma}
\begin{proof}
Since $\mathcal{D}_\ell(q_x,q_y) \setminus W$ is convex, there are two tangents from $(q_x,q_y)$. Let $\Vec{r}_r$ and $\Vec{r}_\ell$ be two rays starting from $(q_x,q_y)$ and representing the right and left tangents, respectively (see Fig.~\ref{fig:actuated-whisking}). Rays $\Vec{r}_r$ and $\Vec{r}_\ell$ intersect the boundary of $\mathcal{D}_\ell(q_x,q_y)$ at two distinct points, denoted by $p_r$ and $p_\ell$, respectively. Let $p_\ell{}^\frown p_r$ be the circular arc from $p_\ell$ to $p_r$. A circular arc $p{}^\frown p'$ satisfies $p''>_\ell p, \forall p''\in p{}^\frown p'$.
Let $A$ (green area in Fig.~\ref{fig:actuated-whisking}) be the circular sector corresponding to the arc $p_\ell{}^\frown p_r$ and let $B$ (blue area in Fig.~\ref{fig:actuated-whisking}) be the connected region bounded by $\Vec{r}_r$ (corresponding to the left half-plane determined by $\Vec{r}_r$), $\Vec{r}_\ell$ (corresponding to the right half-plane determined by $\Vec{r}_\ell$) and $\partial W$ that includes $(q_x,q_y)$.
Then, $\mathcal{D}_\ell(q_x,q_y) \cap V(q_x,q_y) = A \cup B$. 
Consider the \ac{ccw} motion. The area swept before bending is the circular sector determined by the arc $p_0{}^\frown p_1$. Since $p_0 >_\ell p_2$, sensor achieves its nominal form before it reaches back to $p_0$. Then, the total area swept with $\alpha=0$ is the circular sector determined by $p_0{}^\frown p_1$. Furthermore, since $p_1$ is the intersection point when the bending begins, the ray starting at $(q_x,q_y)$ and passing through $p_1$ is the right tangent $\Vec{r}_r$. Similarly, consider the \ac{cw} motion. The area swept before the bending begins is the circular sector determined by the arc $p_3{}^\frown p_0$. Since $p_3$ is before the sensor begins to bend, the ray starting at $(q_x,q_y)$ and passing through $p_3$ corresponds the left tangent $\Vec{r}_\ell$. Hence, $p_3=p_\ell$. Consequently, the union of the two circular sectors corresponding to area swept with $\alpha=0$ resulting from ccw and cw motion equals to $A$. Since $P_{cw} \cap \delW$ and $P_{ccw} \cap \delW$ are two connected curves and $u_1 >_\ell u_2$, part of $\delW$ visible from $(q_x,q_y)$ is swept. Furthermore, bending is a continuous deformation which implies that the swept area does not have any holes. Then, the area swept for $\alpha \neq 0$ equals to $B$. This shows that $A \cup B =\mathcal{D}_\ell(q_x,q_y) \cap V(q_x,q_y)= P$, in which $P=P_{cw}\cup P_{ccw}$.
\end{proof}

\begin{figure}
    \centering
    \includegraphics[scale=0.15]{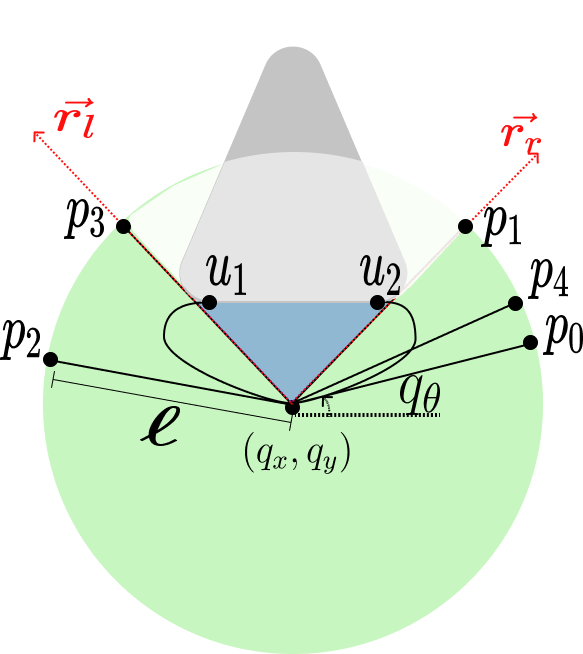}
    \caption{Critical points for an actuated bendable \sensor\ mounted on a revolute joint capable of cw and ccw motions. Light gray area corresponds to $\mathcal{D}_\ell(q_x,q_y) \setminus W$, green area is $A$ and blue area is $B$.}
    \label{fig:actuated-whisking}
\end{figure}

Let $\histX_w$ be the set of all trajectories corresponding to the whisking motion and let $\histX'_w \subseteq \histX_w$ be the ones such that for each $\histx \in \histX'_w$ it is true for the corresponding swept area $P$ that $u_1 >_\ell u_2$ and $p_0>_\ell p_2$ and that $\mathcal{D}_\ell(q_x,q_y) \setminus W$ with respect to the robot configuration $\projFn(\histx(0))=q_r$ is convex. We can now compare a bendable \sensor\, with a depth sensor owing to the inclusion of motion.

\begin{proposition}
Suppose $\histX'_{w}$ is not empty. Let $Q'\subseteq Q$ be defined as $Q' = \bigcup_{\histx \in \histX_{fr}}\projFn(\histx(0)).$ Then, an actuated deflection \sensor\, with a contact detector at the tip executing a whisking motion as defined above can simulate (in the sense from Section \ref{sec:sm}) $h_{od}$ (Model~\ref{model:lim-omnidir-depth}) defined over $Q'$.
\end{proposition}
\begin{proof}
It follows from Lemma~\ref{lemma:bendable_visibility} that for each $q_r\in Q'$, the swept area $P$ with respect to the trajectory $\histx \in \histX'_w$ such that $\projFn(\histx(0))=q_r$ satisfies $P=\mathcal{D}_\ell(q_x,q_y) \cap V(q_x,q_y)$. Then we can define function $y: [0,2\pi) \rightarrow [0,\ell]$ such that $y(\psi)=||(q_x,q_y) - b_{P}(q_x,q_y,q_\theta+\psi)||$, in which $b_{P}(q_x,q_y,q_\theta+\psi)$ is the point struck on the boundary of $\partial P$ by the ray in the direction of $q_\theta+\psi$ starting from $(q_x,q_y)$ if no $x=(q_x,q_y,q_\theta,\psi,\alpha) \in \histx$ satisfies $h'_{\DS}(x)=(0,0)$ (see Eq.~\eqref{eq:comp_deflect} and recall that all the sensor mappings are now defined over the augmented state space), and $\#$ otherwise.
Finally $P$ is defined as Eq.~\eqref{eq:swept_area_def_sensor} for a whisking motion using $h'_{\DS}$ (Model~\ref{model:deflection-sensor}) to determine $\sigma_\alpha$ for each $x \in \histx$.
\end{proof}

Similar to rotary motion, one can consider a translational motion too, in which case the \sensor\, is swept (if in contact) along the boundary of the environment. This introduces an additional complexity since the robot configuration is no longer fixed over the course of sensor movement. We expect translational motion together with a rotational motion (similar to whisking) to be the most useful for touch based navigation. However, we will investigate such touch behaviors in future works.

\section{Discussion}
\label{sec:conc}

In this work, we introduced several noise-free \textit{virtual \sensor\, models} of a \sensor\, mounted on a mobile robot. 
These models were then used to show under which conditions these types of sensors are equivalent to conventional visibility-based models.
The motivation behind this study was to elevate the scattered use of \sensors\ up to the level of a broad sensor category that warrants systematic modeling to better understand their applications for navigation. This would then inspire designing more navigation-centric \sensors\, which could be mounted on a mobile robot for safely exploring unknown environments. This is especially useful in vision devoid scenarios like underwater inspections in murky waters or exploration in areas with suspended caustic chemicals that would otherwise damage sensors like cameras and Lidars. Mobile robots endowed with touch sensing abilities, when deployed in unstructured environments, can benefit from our study in various domains such as health care, sports, ergonomics, logistics, and service robotics.

Existing \sensor\, designs are primarily geared towards furthering our understanding of touch receptors such as vibrissae in rats~\cite{pearson2011biomimetic} and human-like grasping~\cite{chen2018tactile}. However, \sensors\, such as the contact strip (Model~\ref{model:contact-detector-strip}) and the compression sensor (Model~\ref{model:compression-sensor}) are also useful in reducing the uncertainty in state estimate while ensuring the safety of the \sensor\ and the robot.

When comparing \textit{virtual sensor models}, the main difficulty was the lack of shared state space across various sensors under consideration (as certain configurations are not feasible for \sensors). This can be alleviated by designing \sensors\, which are fully compressible, retractable, and/or bendable. To illustrate, consider the scenario shown in Fig.~\ref{fig:compression-issue} where the \sensor\, disrupts the connectivity of the robot configuration space. In such cases, using a different \sensor\, (such as a bendable \sensor) can be the solution. Additionally, for real world deployment, \sensors\, need to be robust to wear and tear and the abilities such as telescopic extensions or fully retractable compression can help the robot maneuver even in tight spaces by modulating its footprint which is otherwise not feasible when using conventional sensors like a camera. 

\begin{figure}[!htbp]
    \centering
    \includegraphics[scale=0.2]{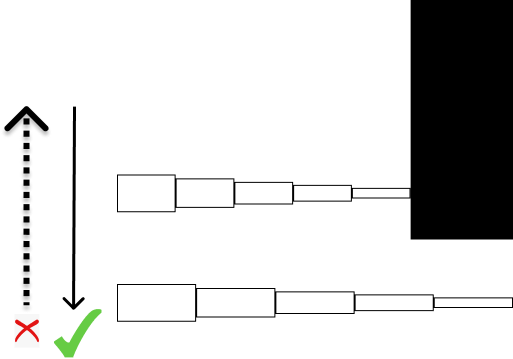}
    \caption{Navigation limitations with compression sensor}
    \label{fig:compression-issue}
\end{figure}

The next step in our study would be to consider actuated \sensors\, on a mobile base such that the system would have both rotational and translation degrees of freedom as opposed to only rotation as considered above. This would open new avenues for developing impact-resilient robots by utilizing \sensor\ feedback-based motion primitives. The models presented here were for an ideal \sensor\, however, the abstract models independent of their physical realization could very well serve a purpose similar to the pinhole camera model for computer vision.  For the real world, sensor noise needs to be accounted for as in such cases the preimages result in a cover as opposed to a partition of the state space. To address the stochastic nature of the sensors, sensor fusion and effective data representation algorithms are critical aspects to consider when using \sensors\ for mobile robot navigation. Sensor fusion techniques could be used for combining spatial, temporal, or spatiotemporal sensor data from multiple modules of the same sensor (as in a vibrissal array) or heterogeneous sensors (as in fusing \sensor\ data with other on board sensors). As a result, developing efficient filtering and sensor fusion algorithms for processing touch data becomes a priority, especially tactile SLAM. Recent attempts at sensor fusion (see~\cite{struckmeier2019vita,pearson2021multimodal}) show promising performance improvements when \sensor\, information is fused with other sensors on board such as cameras which have longer detection range and dense area coverage. Touch sensors, despite their short range, can complement such sensors and dedicated filters for fusing information need to be developed. Aside from accounting for noise, the ideal models considered in this work can also be extended to more practical situations such as those involving multiple forces simultaneously acting on the cantilever, and simultaneous compression and deflecting of the beam.

\bibliographystyle{./bibliography/IEEEtran}
\bibliography{root}

\end{document}